\documentclass{article}

\usepackage{arxiv}

\usepackage{ulem}
\usepackage{amsthm}
\usepackage[utf8]{inputenc} %
\usepackage[T1]{fontenc}    %
\usepackage{hyperref}       %
\usepackage{url}            %
\usepackage{booktabs}       %
\usepackage{amsfonts}       %
\usepackage{nicefrac}       %
\usepackage{microtype}      %
\usepackage{lipsum}		%
\usepackage{epstopdf}
\usepackage{doi}
\usepackage{parskip}
\usepackage{color}
\usepackage{caption}
\usepackage{subcaption}
\usepackage{float}
\usepackage{mathematics}
\usepackage{graphics,graphicx}
\usepackage{wrapfig}
\usepackage{amsmath}
\usepackage{cancel}
\usepackage{enumitem}
\usepackage{todonotes}
\usepackage{soul}
\usepackage{dsfont}
\usepackage{verbatim}
\usepackage{algpseudocode}
\usepackage{algorithm}
\usepackage{mathtools}
\usepackage{tcolorbox}
\usepackage{scalerel}

\algnewcommand\And{\textbf{and}}
\DeclareMathOperator*{\Bigcdot}{\scalerel*{\cdot}{\bigodot}}
\DeclarePairedDelimiter{\ceil}{\lceil}{\rceil}

\newtheorem{lemma}{Lemma}
\newtheorem{remark}{Remark}
\newtheorem{proposition}{Proposition}

\newcommand\includegraphicsifexists[2][width=\linewidth]{\IfFileExists{#2}{\includegraphics[#1]{#2}}{}}

\newcommand{\Lhood}{\mathcal{L}}
\newcommand{\yobs}{\Vec{y}_{obs}}

\renewcommand{\algorithmicrequire}{\textbf{Input:}}
\renewcommand{\algorithmicensure}{\textbf{Output:}}

\usepackage{filecontents}

\title{Accelerating Multilevel Markov Chain Monte Carlo Using Machine Learning Models}

\author{ 
\href{https://orcid.org/0000-0001-6882-9737}{\includegraphics[scale=0.06]{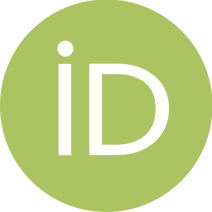}\hspace{1mm}
    Sohail Reddy} \\
	Lawrence Livermore National Laboratory \\
	Liveremore, CA 94550 \\
	\texttt{reddy6@llnl.gov} \\
	\And
	\href{https://orcid.org/0000-0003-3021-0298}{\includegraphics[scale=0.06]{Figures/orcid.png}\hspace{1mm}
	Hillary Fairbanks} \\
	Lawrence Livermore National Laboratory \\
	Liveremore, CA 94550 \\
	\texttt{fairbanks5@llnl.gov} \\
}

\newcommand{\etal}{\textit{et al.}}

\begin{document}
\maketitle

\begin{abstract}

	This work presents an efficient approach for accelerating multilevel Markov Chain Monte Carlo (MCMC) sampling for large-scale problems using low-fidelity machine learning models. While conventional techniques for large-scale Bayesian inference often substitute computationally expensive high-fidelity models with machine learning models, thereby introducing approximation errors, our approach offers a computationally efficient alternative by augmenting high-fidelity models with low-fidelity ones within a hierarchical framework. The multilevel approach utilizes the low-fidelity machine learning model (MLM) for inexpensive evaluation of proposed samples thereby improving the acceptance of samples by the high-fidelity model. The hierarchy in our multilevel algorithm is derived from geometric multigrid hierarchy. We utilize an MLM to acclerate the coarse level sampling. Training machine learning model for the coarsest level significantly reduces the computational cost associated with generating training data and training the model. We present an MCMC algorithm to accelerate the coarsest level sampling using MLM and account for the approximation error introduced. We provide theoretical proofs of detailed balance and demonstrate that our multilevel approach constitutes a consistent MCMC algorithm. Additionally, we derive conditions on the accuracy of the machine learning model to facilitate more efficient hierarchical sampling. Our technique is demonstrated on a standard benchmark inference problem in groundwater flow, where we estimate the probability density of a quantity of interest using a four-level MCMC algorithm.  Our proposed algorithm accelerates multilevel sampling by a factor of two while achieving similar accuracy compared to sampling using the standard multilevel algorithm. 

\end{abstract}


\section{Introduction} \label{sec:Intro}

In the realm of statistical inference, Bayesian methods offer a powerful framework for decision making under uncertainty by integrating prior knowledge with observational data. However, the computational complexity of exact Bayesian inference often becomes prohibitive, particularly for models with high-dimensional parameter spaces or complex likelihood functions. To address this challenge, Markov Chain Monte Carlo (MCMC) methods have emerged as indispensable tools \cite{Robert2021}. MCMC algorithms provide a means to approximate posterior distributions by generating samples from them, circumventing the need for analytical solutions which may be intractable or computationally expensive. This has led to widespread use across many disciplines. 
Although MCMC sampling has alleviated the challenges of analytically computing the posterior distribution, the method’s accuracy and sampling efficiency is still directly governed by the likelihood and the forward model. Employing numerical models for likelihood computations requires high-fidelity solutions, often on a finely resolved space-time discretization, which greatly increases the computational cost of the MCMC sampling. 

To enable MCMC sampling with computationally expensive forward maps, reduced order or surrogate models have been extensively employed \cite{Robert2017}. These surrogate models, in the form of Gaussian processes \cite{Paun2019,Smith2024}, deep neural networks \cite{Yan2020,Hortua2020} and radial basis functions \cite{Orlande2008}, have been employed to reduce the computational cost of the forward model at the expense of accuracy. Physics constrained reduced order models, such as those based on proper orthogonal decomposition \cite{Meixin2021,Charumathi2022}, have been investigated to yield more robust predictions at a fraction of the cost. While surrogate and reduced model greatly reduce the cost of the forward map, the approximation error they introduce can lead to inefficient sampling and errors in the statistics. To address this, delayed acceptance MCMC
, as developed in 
\cite{christen2005markov,efendiev2006preconditioning,Golightly2015,Quiroz2018},
uses a combination of low and high fidelity models whereby the low fidelity model is used to inexpensively evaluate candidate proposals before evaluating with a high fidelity model. Although, this two level approach improves sampling efficiency and accuracy by reducing the impact of approximation error, it does not treat the inefficiency of sampling in high dimensional space. To address this issue, multilevel delayed acceptance (MLDA) sampling approaches based on hierarchical decomposition of the sampling space have become an area of active research \cite{Dodwell2015,Lykkegaard2023,Fairbanks2021}. This approach uses sampling in a low dimensional space to inform high dimensional sampling and has been demonstrated on problems in subsurface flows. The extensive work done on surrogate model accelerated MCMC and multilevel MCMC has naturally led to a coupled approach of machine learning model augmented multilevel MCMC, where surrogate models are used to further accelerate levels in the hierarchy. Approximation errors introduced by inaccurate machine learning models can lead to a divergence of the chain to regions of low posterior densities, thereby ill-informing sampling on subsequent finer levels which leads to significant increase in computational cost and decrease in overall efficiency. Lykkegaard \etal \cite{lykkegaard2021accelerating} used a deep neural network (DNN) to accelerate the lowest-fidelity (coarsest) of a two level hierarchy. They reported difficulty in training the DNN models to predict the forward map with sufficient accuracy, as the coefficient (input vector) dimension was smaller than that of the observational data (output vector) on which they trained their model. They remedied this issue by decomposing the observational domain into subdomains and training a separate DNN model to approximate the forward map on each subdomain.

In this work, we consider a geometric multigrid hierarchy of varying level of refinement and present a multilevel MCMC algorithm that increases the efficiency of the coarsest level by augmenting the hierarchy with a DNN model. %
Our approach utilizes a two stage approach on the coarsest level where the first stage performs an inexpensive Metropolis-Hastings step using the DNN model while the second state performs a filtering step using the finite element forward model. This two stage algorithm prevents divergence of the coarse level chain due to the approximation error of the DNN model and is embedded within a multilevel delayed acceptance algorithm to extend the sampling over all the levels of the multigrid hierarchy. 
The technique is demonstrated on a problem in subsurface flow, namely Darcy flow with uncertain permeability coefficient, whereby both the coefficient and Darcy pressure solution are solved within the mixed finite element framework. Dissimilar to the work of \cite{lykkegaard2021accelerating}, we use a PDE-based approach to generate coefficient data used in the MLM training, as opposed to the commonly used Karhunen-Lo\`{e}ve (KL) expansion~\cite{ghanem2003stochastic}. Consequently, our coefficient (input vector) dimension is large enough to mitigate the need for multiple DNN models to predict the observational output. 

The remainder of the paper is organized as follows. The Bayesian inverse problem is introduced in Section~\ref{sec:MCMC}. This includes a description hierarchical Gaussian random field (GRF) sampling, and our filtered, multilevel MCMC algorithm. We prove that our hierarchical approach preserves detailed balance and is a consistent MCMC algorithm. In Section~\ref{sec:Darcy} we present the problem formulation for Bayesian inference in subsurface flows including the structure of our machine learning model. Finally, in Section~\ref{sec:Results} we present the results of our machine learning accelerated MLMCMC algorithm and compare the filtered algorithm against the unfiltered algorithms. Our results demonstrate a factor of two speed up over the standard PDE-only hierarchy with improved statistics while maintaining high accuracy.

\section{Bayesian Inverse Problem: Markov Chain Monte Carlo} \label{sec:MCMC}

Bayesian inverse problems aim to infer the unknown parameters of a system or model by incorporating prior knowledge and observed data in a probabilistic framework. By leveraging Bayes' theorem, which updates prior information encoded in the prior distribution with observational data, $\yobs$, to obtain posterior distributions. Bayesian inverse problems offer a principled approach to estimating parameters and assessing their uncertainty, thereby facilitating informed decision-making and improving the reliability of predictions. While there exist several approaches, such as maximum-a-posteriori (MAP) estimate \cite{Bassett2018} and variational inference \cite{Blei2018}, for approximating the posterior distribution and quantifying uncertainty in model parameters, we employ MCMC sampling.

This work considers multilevel MCMC sampling from a posterior distribution. To that end, we define a hierarchy of maps $\Set{\mathcal{F}_\ell}_{\ell=0}^L$ with increasing fidelity. Each $\mathcal{F}_\ell := \mathcal{D}_\ell \circ \mathcal{S}_\ell(\zeta_\ell)$ maps the parameters being inferred $\zeta$ to model output $\Vec{y}$ and is a composition of $\mathcal{S}_\ell: \zeta_\ell \mapsto \theta_\ell$ corresponding to the sampler for generating GRFs and $\mathcal{D}_\ell: \theta_\ell \mapsto \Vec{y}_\ell$ corresponding to forward model (e.g. Darcy solver). Although a hierarchy of models can be constructed by considering models with increasing complex physics, we consider it a hierarchy of increasingly refined discretizations of a single PDE, that naturally stem from multigrid methods. Hence, $\ell$ denotes the order of refinement. For the remainder of the manuscript, we denote variables pertaining to the surrogate model by ~$\widehat{\cdot}$.

\subsection{Hierarchical Sampling of Gaussian Random Fields} \label{subsec:MCMC:GRB}
Gaussian random fields have found application in several fields including statistical physics, biomedical imaging, machine learning, and parameterizing material distribution \cite{Lemus2021}. Consider a domain $\Omega \subset \R^d~(d=2,3)$ with a boundary $\Gamma := \partial \Omega$, and a probability space $(\Omega_\mathcal{S},\mathcal{F},\mathbb{P})$ where $\Omega_\mathcal{S}$ is the sample space (the set of all possible events), $\mathcal{F}$ is the $\sigma$-algebra of events, and $\mathbb{P}:\mathcal{F} \rightarrow [0,1]$ is a probability measure. We then seek functions $\cbrac{\theta(\Vec{x},\omega) \in L^2(\Omega) | \Vec{x} \in \Omega, \omega \in \Omega_\mathcal{S}}$ that represent realizations of Gaussian random fields that follow $\theta \sim \mathcal{N}(0,\mathcal{C})$ with covariance $\mathcal{C}$. A popular approach for sampling realizations of GRF involves a KL expansion of the covariance kernel~\cite{ghanem2003stochastic}. Although applicable for lower dimensional problem with a spectrally-fast-decaying kernel, the poor scaling of KL decomposition limits its use for large scale applications. Furthermore, KL expansion is often truncated to a finite number of modes, hence only span a subspace. Efforts to include higher frequency modes in the expansion often leads to numerical artifacts and pollution of the GRFs leading to level-dependent truncated expansion \cite{Teckentrup13}. To address these drawbacks, we sample random fields by solving a stochastic partial differential equation (SPDE) using finite element discretization~\cite{Lindgren11}, and scalable multigrid solvers~\cite{Osborn2017, Osborn2017b, Lee17}. The SPDE approach for sampling GRFs considers the following fractional PDE
\eq{FracSPDE}{
    \rbrac{\kappa^2 - \Delta}^{\alpha/2} \theta(\Vec{x},\omega) = g \mathcal{W}(\Vec{x},\omega), \quad~\Vec{x} \in \R^d,~\alpha = \nu + \dfrac{d}{2},~\kappa > 0,~\nu>0
}
where the scaling factor
\expression{
    g = (4 \pi)^{d/4} \kappa^\nu \sqrt{\dfrac{\Gamma(\nu + d/2)}{\Gamma(\nu)}}
}
imposes unit marginal variance of $\theta(\Vec{x},\omega)$ and the white noise $\mathcal{W}(\Vec{x},\omega)$ is an $L^2(\Omega)$ generalized function satisfying
\expression{
    (\mathcal{W}(\Vec{x},\omega),\phi) \sim \mathcal{N}(0,||\phi||^2_{L^2(\Omega)}),\quad \forall \phi \in L^2(\Omega)
}
Setting $\nu = \frac{1}{2}$ in three dimension, we obtain the standard, integer-order, reaction-diffusion equation 
\subeqs{SPDE}{
    \eq{SPDE:1}{
        \Div{\Vec{\rho}} - \kappa^2 \theta = -g \mathcal{W}(\Vec{x},\omega)
    }
    \eq{SPDE:2}{
        \Vec{\rho} = \Grad{\theta}
    }
}
the solution to which yields realizations of GRFs with exponential covariance function 
\eq{ExpCov}{
    \mathrm{cov}(\Vec{x},\Vec{y}) = \sigma^2 e^{-\kappa \Norm{\Vec{x}-\Vec{y}}}
}
We solve the integer-order PDE using the mixed finite element method and consider the function spaces
\expression{
    \Vec{R} = H(\mathrm{div},\Omega) := \cbrac{\Vec{u} \in \Vec{L}^2(\Omega)~|~\mathrm{div}~\Vec{u} \in L^2(\Omega),\Vec{u}\cdot \Vec{n} = 0~\mathrm{on}~\Gamma}
}
\expression{
    \Theta = L^2(\Omega)
}
where $\Vec{L}^2(\Omega) = [L^2(\Omega)]^d$ is the $d$-dimensional vector function space. 

Denote by $\Omega_h := \bigcup_{i=1}^N K_i$ a partition of $\Omega$ into a finite collection of non-overlapping elements $K_i$ with $h := 
\underset{j=1,\ldots,N}{\max} diam(K_j)$. We take $\Vec{R}_h  \subset H(\mathrm{div},\Omega_h) \subset \Vec{R}$ and $\Theta_h  \subset  L^2(\Omega_h) \subset \Theta$ to be the lowest order Raviart-Thomas and piecewise constant basis functions, respectively. Furthermore, we consider $\Vec{\rho} \in \Vec{R}$ and $\theta \in \Theta$ and denote their finite element representation by $\Vec{\rho}_h \in \Vec{R}_h$ and $\theta_h \in \Theta_h$. Introducing test functions $\Vec{v}_h \in \Vec{R}_h$ and $q_h \in \Theta_h$, the weak, mixed form of \eqref{SPDE} reads
\begin{problem}
    Find $\rbrac{\Vec{\rho}_h,\theta_h} \in \Vec{R}_h \times \Theta_h$ such that
    \aligneq{SPDE:Mixed}{
            \rbrac{\Div{\Vec{\rho}_h},q_h} - \kappa^2 \rbrac{\theta_h,q_h} &= -g \rbrac{\mathcal{W}(\omega),q_h}, \quad &\forall \Vec{v}_h \in \Vec{R}_h \\
            \rbrac{\Vec{\rho}_h,\Vec{v}_h} + \rbrac{\Div{\Vec{v}_h},\theta_h} &= 0, \quad &\forall q_h \in \Theta_h
    }
    with boundary conditions $\Vec{\rho}_h \cdot \Vec{n} = 0$.
\end{problem}
with inner products defined as 
\expression{
    \rbrac{\theta,q} = \integral{\Omega}{}{\theta~q~d\Omega},~\mathrm{and} ~\rbrac{\Vec{\rho},\Vec{v}} = \integral{\Omega}{}{\Vec{\rho}\cdot\Vec{v}~d\Omega}
}
For ease of notation, letting $\zeta_h := \rbrac{\mathcal{W}_h(\omega),q_h} = W_h \mathcal{W}_h = W_h^{1/2} \xi_h$, the resulting system is then given by
\eq{SPDE:Saddle}{
    \smatrix{ M_h & B_h^T \\ B_h & -\kappa^2 W_h } \smatrix{ \Vec{\rho}_h \\ \theta_h } = \smatrix{ \Vec{0} \\ -g\zeta_h }
}
where $\Theta_h \ni \xi_h \sim \mathcal{N}(0,I)$, $M_h$ is the mass matrix corresponding to Raviart-Thomas space $\Vec{R}_h$, $W_h$ is the mass matrix corresponding to the space $\Theta_h$, and $B_h$ is the discrete divergence operator. Solution to \eqref{SPDE:Saddle} with independently sampled $\mathcal{W}_h$ yields a realization of a GRF independent of the other levels in the hierarchy. For multilevel sampling, we require the fine-level white noise $\mathcal{W}_h$  to be conditioned on the coarse-level white noise $\mathcal{W}_{H}$ (i.e. $\mathcal{W}_h | \mathcal{W}_{h < H}$).

Consider two levels $\ell$ and $L$ in a hierarchy and their corresponding independent white noise $\mathcal{W}_\ell \sim \mathcal{N}(0,W^{-1}_\ell)$ and $\mathcal{W}_L  \sim \mathcal{N}(0,W^{-1}_L)$, respectively, with $L > \ell$ implying $h < H$ with nested spaces $\Vec{R}_H \subset \Vec{R}_h$ and $\Theta_H \subset \Theta_h$. Here onwards, we replace the subscripts $h$ and $H$ with $L$ and $\ell$, respectively, with $\mathcal{G}_H \subset \mathcal{G}_h \rightarrow \mathcal{G}_\ell \subset \mathcal{G}_L,~\mathcal{G}\in \cbrac{\Vec{R},\Theta}$.
Furthermore, denote by $\Theta_L \ni \mathcal{W}_\ell^L \sim \mathcal{N}(0,W^{-1}_\ell) $, the realization of white noise $\mathcal{W}_\ell$ on level $L$. Using standard multigrid prolongation operator $P$ and restriction operator $\Pi$, we can write $P \Pi \mathcal{W}_\ell^L = P \mathcal{W}_\ell$. Following the work of \cite{Fairbanks2021}, we then construct a hierarchical decomposition of the white noise on level $L$, conditioned on $\mathcal{W}_\ell$, as 
\eq{MLDecomp}{
    \widetilde{\mathcal{W}}_L = \underbrace{P \mathcal{W}_{\ell}}_{\mathcal{Q}_L \mathcal{W}_\ell^L} + (I - \mathcal{Q}_L)\mathcal{W}_L
}
where $\mathcal{Q}_L:\Theta_\ell \mapsto \Theta_L$ is an $L^2$-projection with $\mathcal{Q}_L = P \Pi$, and $\Pi = W_{\ell}^{-1} P^T W_{L}$ with $\Pi P = I$. The first term on the right ($P \mathcal{W}_{\ell}\sim \mathcal{N}(0,W_\ell)$) represents the prolongation of the coarse level white noise to the fine level, while the second term represents the white noise complement in $\Theta_L \setminus \Theta_\ell$. For ease of implementation, we employ the following form of the decomposition of the source term
\aligneq{HLDecomp}{
    \widetilde{\zeta}_L := W_L\widetilde{\mathcal{W}}_L & = W_L P \mathcal{W}_{\ell} + W_L(I - \mathcal{Q}_L)\mathcal{W}_L \\
                      & = W_L P W^{-1}_\ell \zeta_{\ell} + W_L(I - P  W_{\ell}^{-1} P^T W_{L}) W^{-1}_L \zeta_L \\
                      & = \Pi^T \zeta_{\ell} + (I  - \Pi^T P^T ) \zeta_L \\
                      & = \Pi^T W^{1/2}_\ell \xi_{\ell} + (I  - \Pi^T P^T ) W^{1/2}_L \xi_L \sim \mathcal{N}(0,W_L) \\
}
The decomposition \eqref{HLDecomp} can be applied recursively to finer levels to obtain conditioned white noise, and therefore, realizations of conditioned GRFs on finer levels. Furthermore, note that in \eqref{HLDecomp} the matrix square-root is needed. Since, $\Theta_\ell$ is taken to be piecewise constant, $W_\ell$ is diagonal. We employ mesh-embedding \cite{Osborn2017b} to avoid boundary artifacts arising from no-flux boundary conditions on finite domains.

Following this approach, we are now able to sample fine and coarse realizations of white noise from a joint distribution, which is an essential component of multilevel sampling.

\subsection{Multi-Level Markov Chain Monte Carlo} \label{subsec:MCMC:ML}

Our goal is to estimate the posterior mean of a scalar quantity of interest $Q:=Q(\zeta)$
\expression{
    \mathbb{E}_\pi[Q] = \integral{\Omega}{}{Q(\zeta)\pi(\zeta)~d\zeta}
}
where the posterior distribution is $\pi(\zeta) \equiv \pi(\zeta | \Vec{y}_{obs}) \propto \Lhood(\Vec{y}_{obs}|\zeta) \nu(\zeta)$, with prior $\nu(\zeta)$ and likelihood $\Lhood(\yobs|\zeta)$. We estimate the integral by means of Monte Carlo (MC) sampling by drawing samples from the posterior using MCMC
\expression{
    \mathbb{E}_\pi[Q] \approx \overline{Q} = \dfrac{1}{N} \sum_{i=1}^N Q(\zeta^{(i)})
}
where $\zeta^{(i)} \sim \pi(\zeta | \Vec{y}_{obs})$. Using $\zeta_\ell$ as the discrete approximation of white noise on level $\ell$, sampled according the posterior distribution on level $\ell$, $\pi_\ell(\zeta) \equiv \pi(\zeta_\ell | \Vec{y}_{obs})$, we define the discrete QoI approximation on level $\ell$ as $Q_{\ell}:=Q_\ell(\zeta_\ell)$. Then the multilevel decomposition of the expectation is given as
\eq{ML:Expectation}{
    \mathbb{E}_{\pi_L}[Q_L] = \mathbb{E}_{\pi_0}[Q_0] + \sum_{\ell=1}^{L} \rbrac{\mathbb{E}_{\pi_{\ell}}[Q_{\ell}] - \mathbb{E}_{\pi_{\ell-1}}[Q_{\ell-1}]}.
}
We use the estimator $\widehat{Y}_\ell^{N_\ell}$ for the difference $\mathbb{E}_{\pi_\ell}[Q_\ell] - \mathbb{E}_{\pi_{\ell-1}}[Q_{\ell-1}]$, defined as
\eq{ML:Y}{
    \widehat{Y}_\ell^{N_\ell} = \dfrac{1}{N_\ell} \sum_{i=1}^{N_\ell} Y^{(i)}_\ell = \dfrac{1}{N_\ell} \sum_{i=1}^{N_\ell} \rbrac{Q^{(i)}_{\ell} - Q^{(i)}_{\ell-1}},
}
where $Q^{(i)}_{\ell}:=Q_{\ell}(\zeta_\ell^{(i)})$ is the $i^{th}$ independent sample on level $\ell$, and $N_\ell$ is the number of samples performed on level $\ell$. 

Multilevel Monte Carlo (MLMCM) and Multilevel MCMC (MLMCMC) obtain variance reduction by sampling each $Y_\ell^{(i)}= Q^{(i)}_{\ell} - Q^{(i)}_{\ell-1}$ from a joint distribution, with expectation and variance denoted as $\mathbb{E}_{\pi_\ell,\pi_{\ell-1}}[Y_\ell]$ and $\Var[\pi_\ell,\pi_{\ell-1}]{Y_\ell}$, respectively. As described in Section~\ref{subsec:MCMC:GRB}, we obtain these samples via the hierarchical PDE sampler approach. Let the computational cost of generating an independent sample of $Y_\ell$ be denoted as $\overline{C}_\ell$ for $\ell=0,1,\hdots L$. For a given MSE tolerance $\epsilon^2$, the optimal number of samples on each level, $N_\ell$, is calculated as
\eq{EffectiveSampleSize}{
    N_\ell = \dfrac{2}{\epsilon^2} \rbrac{\sum_{k=0}^L \sqrt{\Var[\pi_k,\pi_{k-1}]{Y_k} \overline{C}_k}} \sqrt{\dfrac{\Var[\pi_\ell,\pi_{\ell-1}]{Y_\ell}}{\overline{C_\ell}}}.
}
If the estimate of $\Var[\pi_\ell,\pi_{\ell-1}]{Y_\ell}$ decays with level refinement, fewer simulations will be needed on the finer levels. We refer the reader to Appendix~\ref{Appen:NEffective} for details on this derivation.

We sample from the posterior using a Multi-Level Delayed Acceptance (MLDA) algorithm and generate proposals $\Vec{\zeta}_\ell^P$ on level $\ell$, conditioned on the current state $\Vec{\zeta}_\ell^C$ and the coarse levels using preconditioned Crank-Nicolson
\expression{
    \Vec{\zeta}_\ell^P := \sqrt{1-\beta^2} \Vec{\zeta}_\ell^C + \beta \widetilde{\zeta}_\ell
}
where $\widetilde{\zeta}_\ell$ is the white noise on level $\ell$ conditioned on the coarse level white noise (\eqref{HLDecomp}), and $\Vec{\zeta}_\ell^P \sim q(\Vec{\zeta}|\Vec{\zeta}_\ell^C, \widetilde{\zeta}_\ell) = \mathcal{N}(\sqrt{1-\beta^2} \Vec{\zeta}_\ell^C, \beta^2 W_\ell)$.
\begin{remark}\normalfont
    The preconditioned Crank-Nicolson proposal $q(\zeta^P|\zeta^C)$ with a prior $\nu \sim \mathcal{N}(0,W)$ satisfies $\nu(\zeta^C)q(\zeta^P|\zeta^C) = \nu(\zeta^P)q(\zeta^C|\zeta^P)$.
\end{remark}
Since the multilevel sampling redistributes the majority of the computational costs over the coarser levels, we further accelerate the coarsest level using a machine learning model. This is acheived by first using the machine learning model to evaluate the generated proposals followed by a filtering stage using the PDE model on the same level to filter out poor proposals before proposing them to the subsequent finer levels. This filtering approach, often referred to as \emph{Surrogate-Transistion} \cite{Liu2008}, is presented in Alg \ref{Alg:FilterMCMC}, where the first stage is the standard Metropolis-Hastings algorithm, hence, is ergodic and satisfies detailed balance. Here, we prove the complete, two-stage algorithm also satisfies detailed balance. 
\begin{algorithm}[h] 
    \caption{Generate a single sample on a level using the filtered Metropolis-Hastings algorithm.} \label{Alg:FilterMCMC}
    \algorithmicrequire{~Current state: $\Vec{\zeta}^C$} \\
    \algorithmicensure{~Next state: $\Vec{\zeta}^\star$}
    \begin{algorithmic}
    \State $\Vec{\zeta}^P \sim q(\Vec{\zeta}|\Vec{\zeta}^C)$ \Comment{Generate proposal} 
    \State $\widehat{\alpha} = \min\cbrac{1,\dfrac{\widehat{\Lhood}(\yobs|\Vec{\zeta}^P)}{\widehat{\Lhood}(\yobs|\Vec{\zeta}^C)}}$ \Comment{Compute acceptance ratio} \\
    \If{$\widehat{\alpha} \geq u \sim \mathcal{U}(0,1)$ } \Comment{Filter proposed sample if accepted}
        \State $\Vec{\zeta}^\star \gets$ \Call{Filter}{$\Vec{\zeta}^C$,$\Vec{\zeta}^P$}
    \Else \Comment{Reject the proposed sample}
        \State $\Vec{\zeta}^\star = \Vec{\zeta}^C$
    \EndIf        
    \Procedure{Filter}{$\Vec{\zeta}^C$,$\Vec{\zeta}^P$} \label{Alg:Proc:Filter}
        \State $\alpha = \min\cbrac{1,\dfrac{\Lhood(\yobs|\Vec{\zeta}^P)}{\Lhood(\yobs|\Vec{\zeta}^C)} \dfrac{\widehat{\Lhood}(\yobs|\Vec{\zeta}^C)}{\widehat{\Lhood}(\yobs|\Vec{\zeta}^P)}}$ \Comment{Compute multi-level acceptance ratio}\\
        \If{$\alpha \geq u \sim \mathcal{U}(0,1)$} \Comment{Accept the proposed sample}
            \State \textbf{return} $\Vec{\zeta}^P$
        \Else \Comment{Reject the proposed sample}
            \State \textbf{return} $\Vec{\zeta}^C$
        \EndIf        
    \EndProcedure
    \end{algorithmic}
\end{algorithm}

\begin{proposition} \label{lemma:Filter}
    Algorithm \ref{Alg:FilterMCMC} simulates a Markov chain that is in detailed balance with $\pi(\cdot)$
\end{proposition}
\begin{proof}
    Let $g(\zeta^P | \zeta^C) = q(\zeta^P|\zeta^C) \widehat{\alpha}(\zeta^P|\zeta^C)$ be the transistion kernel of the first stage that is in detailed balance with $\widehat{\pi}(\cdot)$ 
    \expression{ 
         ~\widehat{\pi}(\zeta^C)g(\zeta^P|\zeta^C) =  ~ \widehat{\pi}(\zeta^P)g(\zeta^C|\zeta^P)
    }
    Then the second stage has the transistion kernel 
    \expression{
        G(\zeta^P | \zeta^C) = g(\zeta^P | \zeta^C)~\min\cbrac{1, \dfrac{\pi(\zeta^P)}{\pi(\zeta^C)}\dfrac{\widehat{\pi}(\zeta^C)}{\widehat{\pi}(\zeta^P)}}
                             = g(\zeta^P | \zeta^C)~\min\cbrac{1, \dfrac{\Lhood(\yobs|\zeta^P)}{\Lhood(\yobs|\zeta^C)}\dfrac{\widehat{\Lhood}(\yobs|\zeta^C)}{\widehat{\Lhood}(\yobs|\zeta^P)}}
    }
    and satisfies the detailed balance condition
    \aligneq{DetailedBalance:Filter}{
        \pi(\zeta^C) G(\zeta^P | \zeta^C) &= \pi(\zeta^C) g(\zeta^P | \zeta^C)~\min\cbrac{1, \dfrac{\pi(\zeta^P)}{\pi(\zeta^C)}\dfrac{\widehat{\pi}(\zeta^C)}{\widehat{\pi}(\zeta^P)}} \\
                                          &= \dfrac{\widehat{\pi}(\zeta^P) g(\zeta^C | \zeta^P) \pi(\zeta^C)}{\widehat{\pi}(\zeta^C)}~\min\cbrac{1, \dfrac{\pi(\zeta^P)}{\pi(\zeta^C)}\dfrac{\widehat{\pi}(\zeta^C)}{\widehat{\pi}(\zeta^P)}} \\
                                          &= \widehat{\pi}(\zeta^P) g(\zeta^C | \zeta^P)~\min\cbrac{\dfrac{\pi(\zeta^C)}{\widehat{\pi}(\zeta^C)},\dfrac{\pi(\zeta^P)}{\widehat{\pi}(\zeta^P)}} \\
                                          &= \pi(\zeta^P) G(\zeta^C | \zeta^P)
    }
\end{proof}

\begin{remark}\normalfont
    As the accuracy of the surrogate model tends to that of the true model (i.e. $\widehat{\mathcal{F}}(\Vec{\zeta})\rightarrow \mathcal{F}(\Vec{\zeta})$) then $\widehat{\Lhood}(\yobs|\Vec{\zeta})\rightarrow \Lhood(\yobs|\Vec{\zeta})$, and the acceptance ratio of the filter $\alpha \rightarrow 1$. Hence, for $\widehat{\mathcal{F}}(\Vec{\zeta}) = \mathcal{F}(\Vec{\zeta})$, the filter is idempotent and the two-stage algorithm is equivalent to a single-stage Metropolis-Hastings algorithm.
\end{remark}

Since the first stage performs an inexpensive evalution of the proposed sampled with $\widehat{\mathcal{F}}$ followed by a filtering stage with $\mathcal{F}$, it is desirable for the acceptance probability of the second stage to be greater than (or equal to) that of the first stage. 
The filtered-MCMC sampling of the coarsest level is then coupled to the remaining, finer levels of the hierarchy using the standard multi-level delayed acceptance (MLDA) algorithm. This coupled algorithm, given in Alg. \ref{Alg:HMCMC}, further accelerates the complete MLDA hierarchical sampling. Furthermore, this couple sampling satisfies the properties of the MCMC samplers, as demonstrated by the following lemma. 
\begin{algorithm}[h]
    \caption{Generate a single sample on the fine level $L$ conditioned on the coarse level $\ell$.} \label{Alg:HMCMC}
    \algorithmicrequire{~Current state: $\Vec{\zeta}_L^C$, $\Vec{\zeta}_\ell^C$} \\
    \algorithmicensure{~Next state: $\Vec{\zeta}_L^\star$}
    \begin{algorithmic}
    \State Generate $\Vec{\zeta}_\ell^\star$ using (filtered) Metropolis-Hasting using $\Vec{\zeta}_\ell^C$ \Comment{Generate sample on the coarse level}
    \State Compute $\widetilde{\zeta}_L$ using \eqref{HLDecomp} and $\Vec{\zeta}_\ell^\star$ \Comment{Generate coarse-level conditioned white noise}
    \State $\Vec{\zeta}_L^P \sim q(\Vec{\zeta}_L|\Vec{\zeta}_L^C,\widetilde{\zeta}_L)$ \Comment{Generate fine-level proposal} 
    \State $\alpha_{ML} = \min\cbrac{1,\dfrac{\Lhood_L(\yobs|\Vec{\zeta}_L^P) \Lhood_\ell(\yobs|\Vec{\zeta}_\ell^C)}{\Lhood_L(\yobs|\Vec{\zeta}_L^C)\Lhood_\ell(\yobs|\Vec{\zeta}_\ell^\star)}}$ \Comment{Compute multilevel acceptance ratio} \\
    \If{$\alpha_{ML} \geq u \sim \mathcal{U}(0,1)$} \Comment{Accept proposed sample}
        \State $\Vec{\zeta}_L^\star = \Vec{\zeta}_L^P$
    \Else \Comment{Reject the proposed sample}
        \State $\Vec{\zeta}_L^\star = \Vec{\zeta}_L^C$
    \EndIf        
    \end{algorithmic}    
\end{algorithm}

\begin{lemma}
    Algorithm \ref{Alg:HMCMC} with coarse sampling of $\pi_\ell(\cdot)$ by Alg. \ref{Alg:FilterMCMC} simulates a Markov chain that is in detailed balance with $\pi_L(\cdot)$
\end{lemma}
\begin{proof}
    Since the coarse level sampling is in detailed balance with $\pi_\ell(\cdot)$ by Proposition \ref{lemma:Filter}, the remainder of the proof follows from \cite[Lemma 2.5]{Lykkegaard2023}.
\end{proof}

\begin{remark}\normalfont
    Since the Metropolis-Hasting (first stage of Alg. \ref{Alg:FilterMCMC}), the filtered Metropolis-Hasting (Alg. \ref{Alg:FilterMCMC}) and the MLDA (Alg. \ref{Alg:HMCMC}) algorithms all satisfy detailed balance and, therefore, detailed balance on each level, and are ergodic on each level, the posterior distributions on levels will converge asymptotically to the stationary distribution. Furthermore, all levels sampled using filtered Metropolis-Hastings and PDE based hierarchy will converge to the same stationary distribution; this is not guarenteed for unfiltered sampling approach. 
\end{remark}

\section{Subsurface Flow: Darcy's Equations} \label{sec:Darcy}
We demonstrate our surrogate-augmented MLMCMC approach on a typical problem in groundwater flow described by Darcy's Law, which has been employed extensively in verifying the efficiency of delayed acceptance MCMC algorithms. The governing equations with the appropriate boundary conditions are written as
\subeqs{GovEq}{
\begin{align}
    \spliteq{Darcy}{    
        \Vec{u}(\Vec{x}) + k(\Vec{x}) \Grad{p(\Vec{x})} = \Vec{f}& \quad \mathrm{in~} \Omega 
    }\\
    \spliteq{Mass}{
        \Div{\Vec{u}(\Vec{x})} = 0& \quad \mathrm{in~} \Omega
    }\\
    \spliteq{BC:P}{
        p = p_s & \quad \mathrm{on~} \Gamma_D
    }\\
    \spliteq{BC:NoFlux}{
    \Vec{u}\cdot\Vec{n} = 0 & \quad \mathrm{on~} \Gamma_N
    }
\end{align}
}
with $\Gamma := \Gamma_D \bigcup \Gamma_N = \partial \Omega$ where \eqref{Mass} is the incompressibility condition. The permeability coefficient is assumed to follow log-normal distribution $k(\Vec{x}) = \exp(\theta(\Vec{x},\omega))$ where $\theta(\Vec{x},\omega)$ is a realization of a Gaussian random field with covariance kernel defined by \eqref{ExpCov}.
We consider $\Vec{u} \in \Vec{R}$ and $p \in \Theta$ and denote their finite element representation by $\Vec{u}_h \in \Vec{R}_h$ and $p_h \in \Theta_h$. Then the weak, mixed form of \eqref{GovEq} reads
\begin{problem}
    Find $\rbrac{\Vec{u}_h,p_h} \in \Vec{R}_h \times \Theta_h$ such that
    \aligneq{GovEq:Mixed}{
            \rbrac{k^{-1} \Vec{u}_h,\Vec{v}_h} - \rbrac{\Div{\Vec{v}_h},p_h} &= \rbrac{\Vec{f}_h,\Vec{v}_h}, \quad &\forall \Vec{v}_h \in \Vec{R}_h \\
            \rbrac{\Div{\Vec{u}_h},q_h} &= 0, \quad &\forall q_h \in \Theta_h
    }
\end{problem}
The resulting saddle point system is then given by 
\eq{GovEq:Saddle}{
    \smatrix{ M_h(k) & B_h^T \\ B_h & \Vec{0} } \smatrix{ \Vec{u}_h \\ p_h } = \smatrix{ \rbrac{\Vec{f}_h,\Vec{v}_h} \\ \Vec{0} }
}
We consider the quantity of interest to be the integrated flux across $\Gamma_{out} \subset \Gamma$
\eq{MassFlux}{
    Q = \dfrac{1}{\abs{\Gamma_{out}}} \integral{\Gamma_{out}}{}{\Vec{u}(\cdot,\omega)\cdot \Vec{n} ~d\Gamma}
}

\section{Results} \label{sec:Results}

\subsection{Problem Formulation} \label{subsec:Results:Problem}

We benchmark our approach on a common inference problem in groundwater flow by estimating the posterior distribution of the mass flux through a domain boundary. In particular we consider a unit cube discretized using a structured hexahedral mesh with different spatial resolution at each level in the hierarchy. In keeping with previous multilevel MCMC studies \cite{Fairbanks2021}, we consider a four-level hierarchy, a spatial resolution 16 elements in each direction on the coarsest level and a uniform refinement factor of two. Hence the hierarchy consists of 4k, 32.7k, 262k, and 2.1M elements on levels $\ell_0$, $\ell_1$, $\ell_2$ and $\ell_3$, respectively\footnote{The exact numbers are 4096, 32768, 262144, and 2097152 elements on levels $\ell_0$, $\ell_1$, $\ell_2$ and $\ell_3$, respectively.}. The benchmarking is performed against reference distributions and statistics obtained using the PDE-only hierarchy. The hierarchical PDE sampling and the forward Darcy problem is solved using the scalable, open-source library ParELAGMC, which uses MFEM~\cite{mfem-library} to generate the fine grid finite element space, and {\it hypre}~\cite{hypre} to implement the parallel linear algebra. Each simulation is performed using hybridization algebraic multigrid~\cite{Lee17, dobrev2019algebraic}, which is solved with conjugate gradient preconditioned by {\it hypre}'s BoomerAMG. The machine learning capabilities (training and deployment) are adopted from TensorFlow~\cite{tensorflow2015-whitepaper}.

We analyze two MLM-augmented hierarchical approach: without filter and with filtering. A total of five independent chains are generated for each reference and MLM-augmented approaches and run each chain for 10k samples on the finest ($\ell_3$) level. We consider a prior of Gaussian random fields with correlation length $\lambda = 0.3$ and a marginal variance $\sigma^2 = 0.5$. We generate synthetic observational data $\yobs \in \R^{25}$ on a reference mesh of 16.7M elements and assume that $ \mathcal{F}(\Vec{\zeta}) \sim \mathcal{N}(\yobs,\sigma_\eta I)$ where $I$ is an identity matrix of appropriate dimension and take the noise to be $\sigma_\eta = 0.005$; consequently, the likelihood function has the relation $\Lhood(\yobs|\Vec{\zeta})\propto \exp (-\Vert \mathcal{F}(\Vec{\zeta})-\yobs \Vert^2/2\sigma_\eta^2)$.

Since the number of parameters and the training cost for fully-connected neural networks increases rapidly with the dimensionality of the problem, we employ convolutional neural networks (CNN). For most problem, it is understood that the size of the network typically grows with the dimensionality of the mapping, and the size of the training data required typically grow with the size of the model. Hence, approximating higher levels of the hierarchy with neural networks becomes infeasible due to the additional cost of training and memory requirements for generating and storing high-resolution training data. Therefore, we only apply the surrogate to approximate the maps on the coarsest level, $\widehat{\mathcal{F}}_0$, where~$\widehat{\cdot}$~denotes variables pertaining to the surrogate model. It should be emphasized that although we employ a surrogate model based on (convolutional) neural networks, alternative surrogate or reduced order modeling techniques such as Gaussian process models \cite{Marrel2024}, and linear-subspace \cite{Gooijer2021} and nonlinear manifold \cite{Lee2020} reduced order models can also be used.

The convolutional neural network consists of two convolutional layers, each with ReLU activation function and 16 and 32 filters, respectively. These convolutional layers are followed by two dense layers of with 80 and 25 nodes, respectively. The CNN is trained using 10,000 independently, identically distributed (i.i.d) samples generated using the PDE samples on the coarsest level while a further 5,000 samples are used for validation. The training is performed using ADAM optimizer for 800 epochs with a decaying learning rate and an additional 200 epochs, with a smaller learning rate are used to fine tune the CNN model. Figure \ref{fig:4Levels:DNNLoss} displays the mean-squared error (MSE) for the training and validation sets, and suggests a relatively accurate trained model with MSE on the order of $\mathcal{O}(10^{-3})$.
\begin{figure}[h] \centering
\includegraphicsifexists[width=0.4\textwidth]{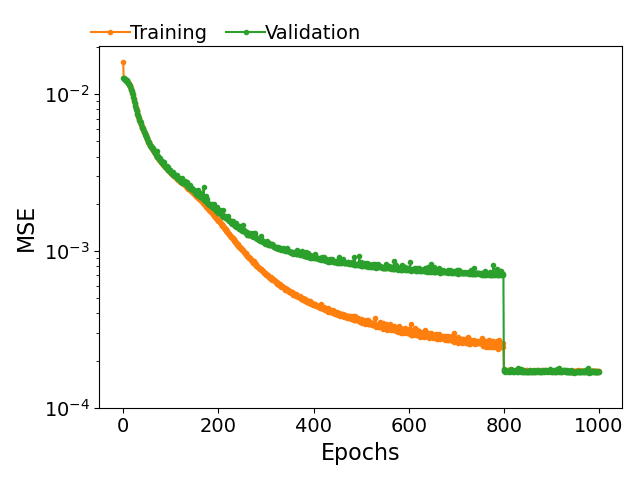} \caption{The convergence in mean-squared error (MSE) for training and validation set for the training epochs.}\label{fig:4Levels:DNNLoss}
\end{figure}

\subsection{Bayesian Inference in Subsurface Flow} \label{subsec:Results:results}

The efficiency and accuracy of the MLM-augmented algorithms can be analyze using metrics such as acceptance rate, number of effective samples, and the speed-up relative to the reference configuration. Figure \ref{fig:4Levels:Metrics} shows these metrics on each level as well as the mean and variance of $Y_\ell$ on each level. Although the mean and variance in $Y_\ell$ on the coarsest level (note that $Y_0 = Q_0$) are comparable for all methods, the variance on the subsequent coarse level greatly increases when using the unfiltered MCMC. Hence, the number of effective samples required on the coarsest level for a certain tolerance ($\epsilon^2 = 0.01$), given by \eqref{EffectiveSampleSize}, significantly increases (Fig. \ref{fig:4Levels:NEffec}). The mean and variance decay of the filtered approach are comparable to the reference configuration and, as assumed in the MLMC theory, follow a log-linear trend.
Although the acceptance rate (Fig. \ref{fig:4Levels:Acceptance}) on the coarsest level with the unfiltered approach is comparable to that of the reference configuration, the acceptance rate on the adjacent level decreases significantly, and indicates that the coarse level sampling with ML forward map (without filtering) generates poor candidate proposals. The acceptance rate on the remaining levels recover to those similar to the reference configuration. The acceptance rate of the filtered algorithm is lower compared to the other algorithms but the filtering stage filters out these poor candidates and drastically improves the acceptance rate of the more computationally expensive second-stage. This improved acceptance rate and the lower cost-per-sample reduces the computational cost of the complete hierarchical sampling. Figure \ref{fig:4Levels:SpeedUp} shows the speed up relative to the reference configuration for a hierarchy with different number of levels. As expected, in both cases, as the number of levels increases the overall efficiency gain decreases, with the unfiltered configuration demonstrating a faster decay. The four level MCMC sampling using our machine learning-augmented filtered MCMC approach results in a factor of two speed up over the standard multilevel PDE hierarchy. Previous work by the authors \cite{Fairbanks2021} have shown a factor of 3.5 speed up over the single level MCMC, hence, our MLM-augmented approach can deliver a speed of a factor of 7 over the single level approach. 

\begin{figure}[h] \centering
\begin{subfigure}[h]{0.32\textwidth} %
\includegraphicsifexists[width=\textwidth]{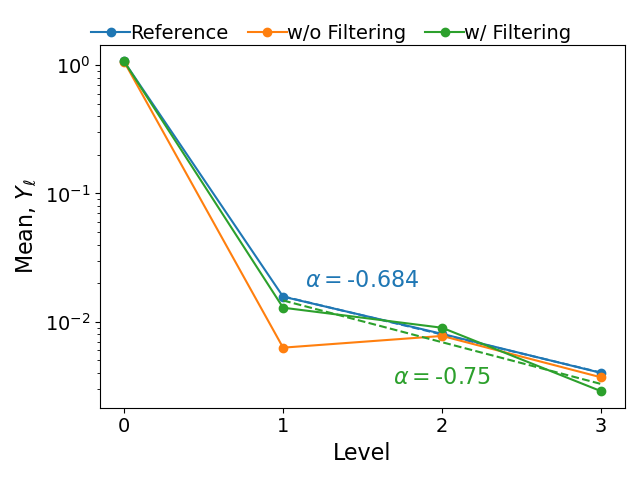} \caption{}\label{fig:4Levels:Mean}
\end{subfigure}
\begin{subfigure}[h]{0.32\textwidth} %
\includegraphicsifexists[width=\textwidth]{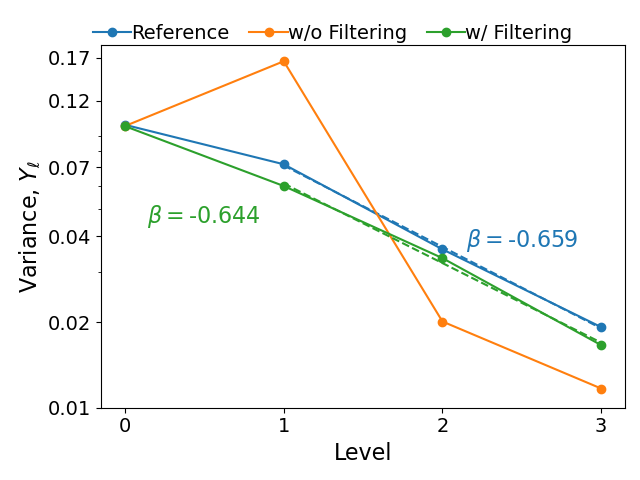} \caption{}\label{fig:4Levels:Variance}
\end{subfigure}
\begin{subfigure}[h]{0.32\textwidth} %
\includegraphicsifexists[width=\textwidth]{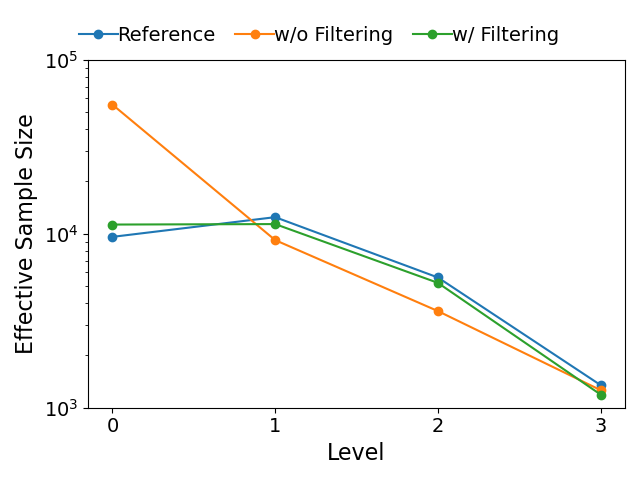} \caption{}\label{fig:4Levels:NEffec}
\end{subfigure}
\begin{subfigure}[h]{0.32\textwidth} %
\includegraphicsifexists[width=\textwidth]{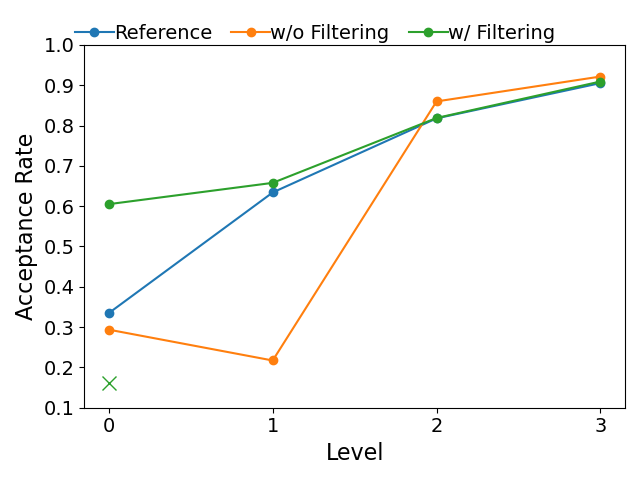} \caption{}\label{fig:4Levels:Acceptance} 
\end{subfigure}
\begin{subfigure}[h]{0.32\textwidth} %
\includegraphicsifexists[width=\textwidth]{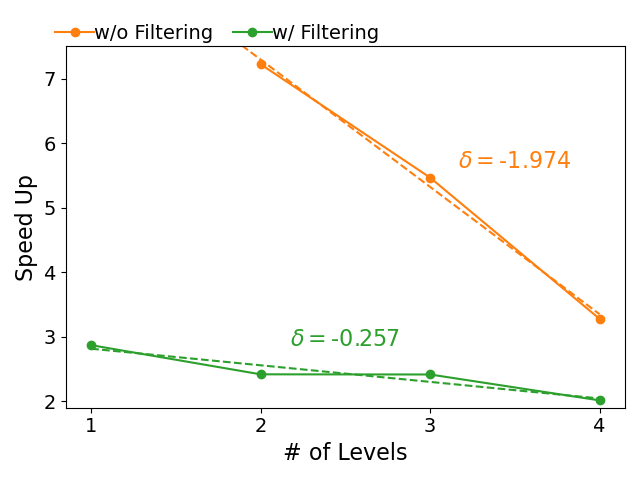} \caption{}\label{fig:4Levels:SpeedUp}
\end{subfigure}
\captionsetup{singlelinecheck=off,font=footnotesize}
\caption[]{Analysis of chains on each level of a four-level hierarchy showing: (a) the mean in $Y_\ell$, (b) the variance in $Y_\ell$, (c) the effective sample size, and (d) the acceptance ratio. (e) The speed-up relative to the reference for a hierarchy with different number of levels. The green $\times$ represents the acceptance of the complete two-stage algorithm, whereas green $\Bigcdot$ represents the acceptance rate of only the second stage.}
\label{fig:4Levels:Metrics}
\end{figure}

We define a measure of distance between probability distributions using the Wasserstein distance. The Wasserstein distance for $Q$ and $Y$ shown in Figs. \ref{fig:4Levels:Wass:Q} and \ref{fig:4Levels:Wass:Y} show that the filtering approach yeild PDF \emph{nearer} to the reference. Furthermore, we see the distance between PDFs of $Q$ increase with increasing levels in the hierarchy, while the distance in $Y$ is relatively constant. It is expected that the distance between $Y$ will decrease with increasing refinement as the discrete solution tends towards a continuous, true solution. The large deviation of the posterior PDF of $Y$ on the coarsest level (Fig. \ref{fig:PDF:Y:L1}) using the unfiltered sampling is clearly evident; the deviation in PDF of $Q$ around the high density region is also visible. The inaccuracies, due to the inaccurate coarse-level forward map, leads the sampling of regions of low posterior densities, which then propogate to the finer levels as seen in Figs. \ref{fig:PDF:Q:L3} and \ref{fig:PDF:Y:L3}. %

\begin{figure}[h] \centering
\begin{subfigure}[h]{0.32\textwidth} %
\includegraphicsifexists[width=\textwidth]{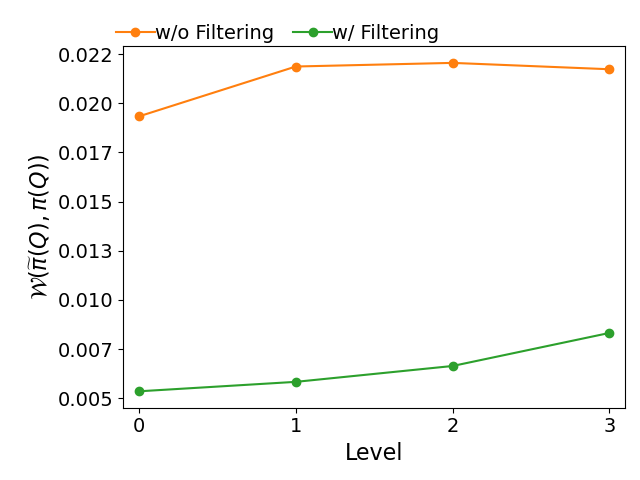} \caption{}\label{fig:4Levels:Wass:Q} 
\end{subfigure}
\begin{subfigure}[h]{0.32\textwidth} %
\includegraphicsifexists[width=\textwidth]{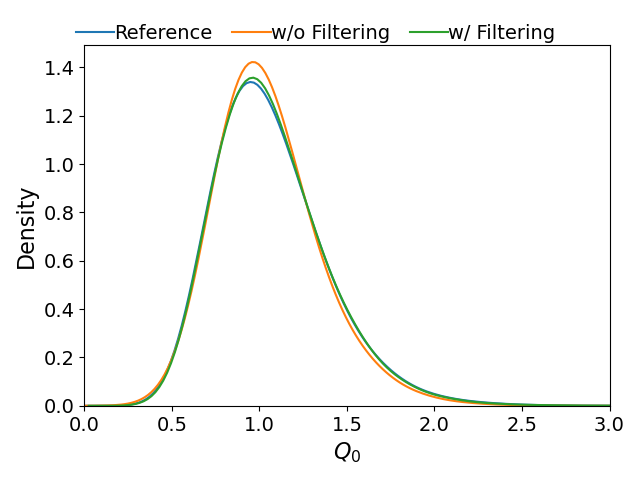} \caption{}\label{fig:PDF:Q:L0} 
\end{subfigure}
\begin{subfigure}[h]{0.32\textwidth} %
\includegraphicsifexists[width=\textwidth]{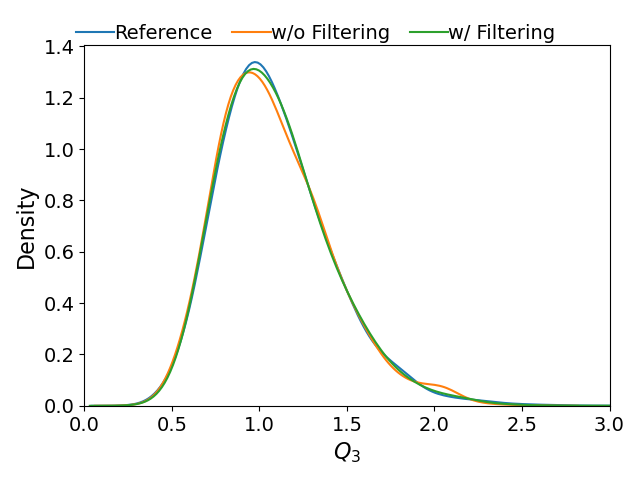} \caption{}\label{fig:PDF:Q:L3} 
\end{subfigure}
\begin{subfigure}[h]{0.32\textwidth} %
\includegraphicsifexists[width=\textwidth]{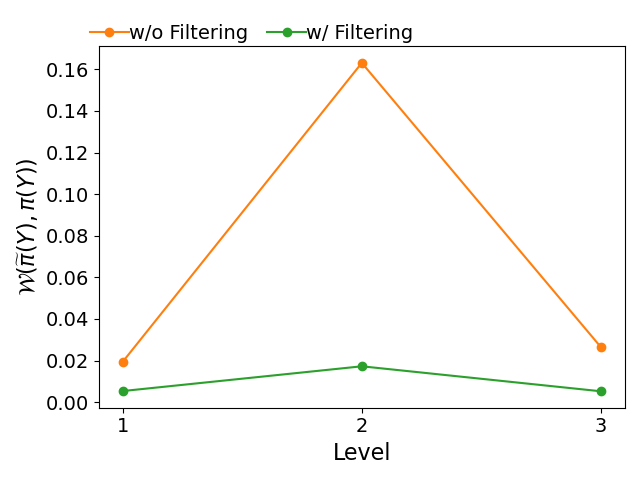} \caption{}\label{fig:4Levels:Wass:Y} 
\end{subfigure}
\begin{subfigure}[h]{0.32\textwidth} %
\includegraphicsifexists[width=\textwidth]{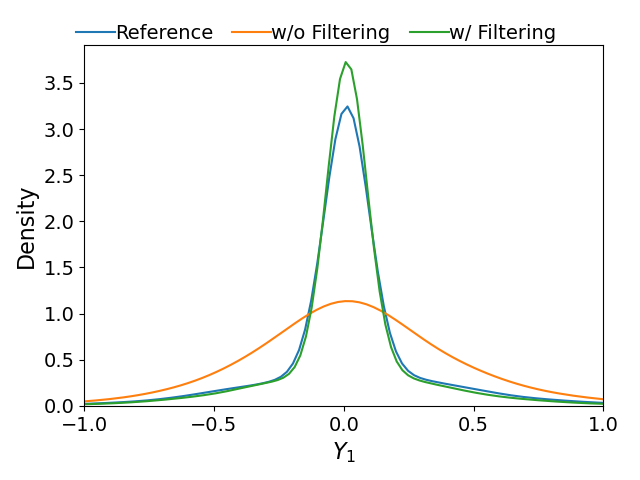} \caption{}\label{fig:PDF:Y:L1} 
\end{subfigure}
\begin{subfigure}[h]{0.32\textwidth} %
\includegraphicsifexists[width=\textwidth]{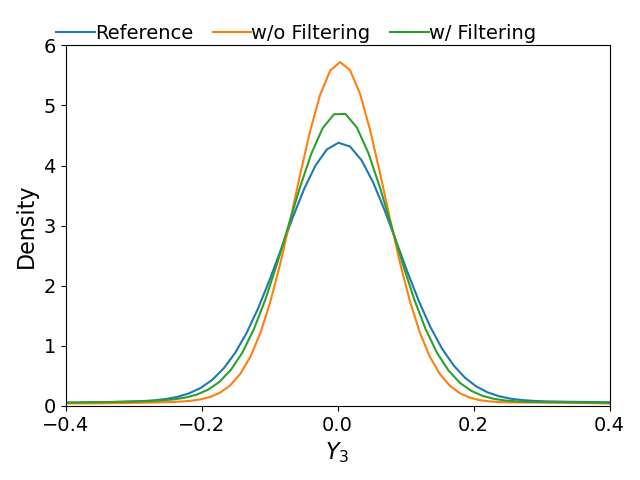} \caption{}\label{fig:PDF:Y:L3} 
\end{subfigure}
\captionsetup{singlelinecheck=off,font=footnotesize}
\caption[]{The Wasserstein distance (a,d) between the reference posterior distribution and those obtained using ML-augmented MCMC on each level, the posterior distributions on the coarsest level (b,e) and the finest level (c,f) for $Q$ (top) and $Y$ (bottom).}
\label{fig:4Levels:WassPDFs}
\end{figure}

The autocorrelation is indicative of the \emph{mixing} of the chains, and the integrated autocorrelation time (IACT) given by \eqref{IACT} represents the number of samples required to generate an independent sample. Figure \ref{fig:4Levels:AutoC:QY} presents the autocorrelation estimates (according to \eqref{NormAutoCorr}) for $Q$ and $Y$ on different levels. The two-stage algorithm displays slower mixing than the reference, and hence, suggests that a significant number of samples are being rejected by the first stage. The second stage displays similar mixing on the coarsest $\ell=0$ level to the other configurations. However, the coarse level sampling using an inaccurate forward map without filtering leads the chain to areas of low posterior density and leads to poor mixing on the subsequent finer levels. This poor mixing propogates to the more computationally-expensive finer level, and its influence is evident even on the finest level. The filtered approach displays similar mixing as the reference configuraton on all levels. Table \ref{table:4Levels:IACT} shows the integrated autocorrelation time (IACT) and cost (in seconds) for generating an independent sample. Although the cost per-independent sample of the unfiltered approach is significantly lower, the additional number of samples needed lead to a decrease in efficiency. The cost for the filtered approach is approximately 1/3 the cost of the reference, and combined with the accuracy of the filtered approach, leads to an increase in efficiency. It is important to note that the speed up presented in Fig. \ref{fig:4Levels:SpeedUp} is for a chain of fixed length (10k samples on the finest level) and demonstrate a factor of 3.5 and factor 2 speed up due to the unfiltered and filtered approaches, respectively, over the reference approach. Table \ref{table:4Levels:IACT} presents the theoretical speed up when considering the number of effective samples needed to tolerance $\epsilon^2 = 0.01$ (see Appendix \ref{Appen:NEffective}) given the cost per independent samples and variance ($\mathbb{V}[Y_\ell]$). Here, we see the divergence of the coarse level chain in the unfiltered approach leads to an overall decrease in efficiency, while the filtered approach still yields a theoretical speed of a factor of 1.2.

\begin{figure}[h] \centering
\begin{subfigure}[h]{0.32\textwidth} %
\includegraphicsifexists[width=\textwidth]{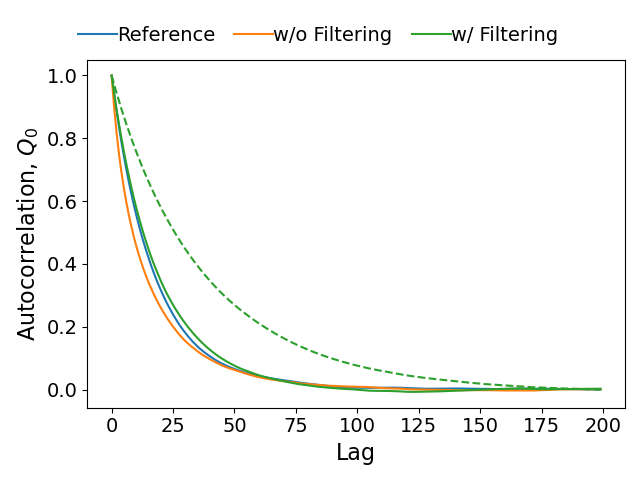} \caption{} \label{fig:AutoC:Q:L0}
\end{subfigure}
\begin{subfigure}[h]{0.32\textwidth} %
\includegraphicsifexists[width=\textwidth]{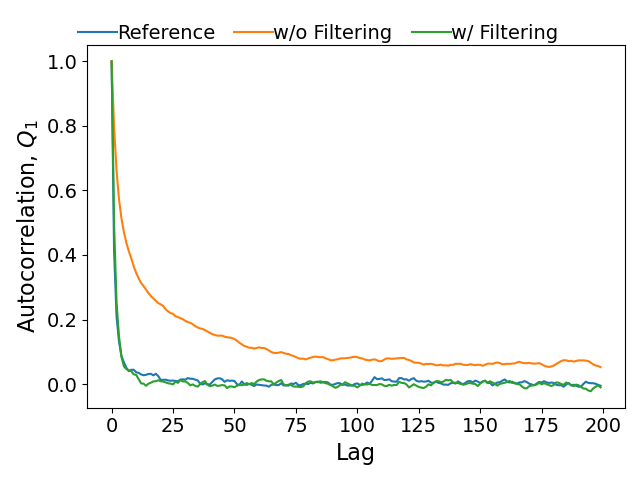} \caption{} \label{fig:AutoC:Q:L1} 
\end{subfigure}
\begin{subfigure}[h]{0.32\textwidth} %
\includegraphicsifexists[width=\textwidth]{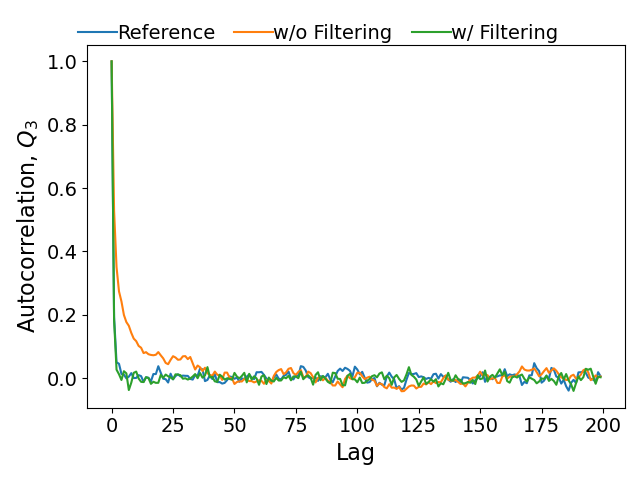} \caption{} \label{fig:AutoC:Q:L3} 
\end{subfigure}
\begin{subfigure}[h]{0.32\textwidth} %
\begin{tcolorbox}[width=\textwidth,opacityfill=0]
\end{tcolorbox}
\end{subfigure}
\begin{subfigure}[h]{0.32\textwidth} %
\includegraphicsifexists[width=\textwidth]{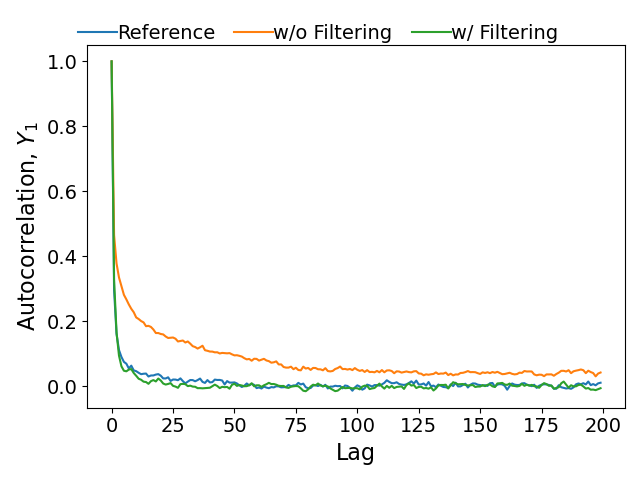} \caption{} \label{fig:AutoC:Y:L1} 
\end{subfigure}
\begin{subfigure}[h]{0.32\textwidth} %
\includegraphicsifexists[width=\textwidth]{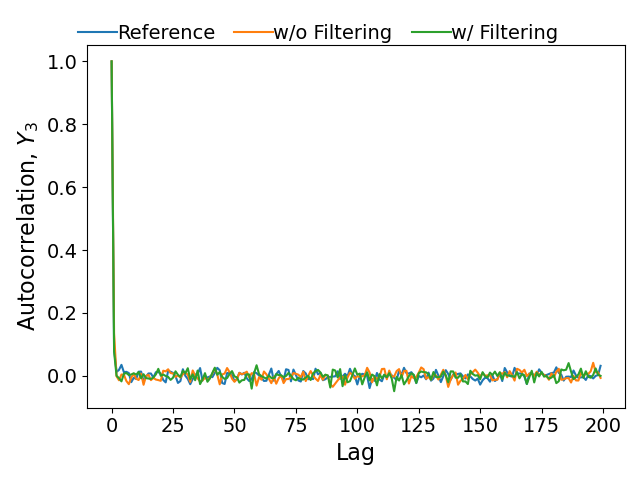} \caption{} \label{fig:AutoC:Y:L3} 
\end{subfigure}
\captionsetup{singlelinecheck=off,font=footnotesize}
\caption[]{Autocorrelation estimates of quantity of interest $Q$ (top) and corresponding $Y$ (bottom) on different levels for increasing lag times. In Fig. 
\ref{fig:AutoC:Q:L0}, the solid line represent the estimate for only the second stage, while the dashed line represents the estimate of the complete two-stage algorithm.}
\label{fig:4Levels:AutoC:QY}
\end{figure}

\begin{table}[h]
\caption{Average (over five chains) IACT of $Y$, the cost per independent sample, in seconds, on each level of a four-level hierarchy, total cost and the speed up. Note that $Y_0 = Q_0$; the value in the bracket represents the IACT of the second-stage.} \label{table:4Levels:IACT}
\centering
\begin{tabular}{ l | c c c c | c c c c | c c} 
\hline 
\hline 
 & \multicolumn{4}{c|}{IACT} & \multicolumn{4}{c|}{Cost per Independent Sample} & Total Cost & Speed Up \\
\hline
Models  &  $\ell_0$ & $\ell_1$ & $\ell_2$ & $\ell_3$ &  $\ell_0$ & $\ell_1$ & $\ell_2$ & $\ell_3$ & & \\ 
\hline 
Reference  & 38 & 5.5  & 3.6 & 3.1 & 1.46 & 0.63  & 1.55 & 14.3 & 49803 & 1.0 \\
Unfiltered & 33 & 49.4 & 4.0 & 3.1 & 0.07 & 4.53 & 3.25 & 14.66 & 75655 & 0.66 \\
Filtered   & 77 (39) &5.4 &  3.4 & 3.1 & 0.91 & 0.58 &  1.49 & 14.08 & 41386 & 1.20 \\
\hline 
\hline
\end{tabular}
\end{table}

\begin{remark}\normalfont
    Similar to \cite{Fairbanks2021}, each $\widehat{Y}_\ell$ estimate uses only two levels $\ell$ and $\ell-1$, whereas \cite{Dodwell2015} uses a hierarchical sampling on all coarser levels for the estimation. %
\end{remark}

\section{Conclusion} \label{sec:Conclusion}

This paper introduces an efficient approach for multilevel Markov Chain Monte Carlo (MCMC), by leveraging a geometric multigrid hierarchy and augmenting the coarse level with a machine learning model (MLM). By integrating machine learning model on the coarsest level, we significantly reduce the computational cost associated with generating training data and training the model. The MLM allows for a computationally efficient evaluation of coarse level proposed samples and limits propagation of poor samples to the computationally more expensive finer levels. Rather than simply replacing a numerical model with a machine learning model, which introduces approximation errors, we introduce an additional filtering stage on the coarsest level to circumvent this approximation error and limit proposal of these poor samples to the finer levels. We prove that our filtered MCMC algorithm satisfies detailed balance and is a consistent MCMC sampling algorithm. Furthermore, we derive conditions on the accuracy of the MLM, compared to the reference model, to yield more efficient sampling. We demonstrate the practical applicability of our technique on a large-scale problem in groundwater flow and report a factor of two speedup over the reference method, while maintaining high levels of accuracy. The posterior densities and the statistical moments obtained using our filtered MCMC algorithm were in good agreement with those obtained using the reference sampler.  Furthermore, both our algorithm and the standard multilevel method exhibit similar decay in mean and variance that align well with theoretical expectation. These findings underscore the robustness, efficiency, and theoretical consistency of our proposed multilevel MCMC approach, offering significant reducation in computational cost and increase in efficiency.

\section*{Acknowledgements}

This work was supported by U.S. Department of Energy's Advanced Scientific Computing Research (ASCR) under award 33825/SCW0897 and was performed under the auspices of the U.S. Department of Energy by Lawrence Livermore National Laboratory under Contract DE-AC52-07NA27344.

\bibliographystyle{unsrt}
\bibliography{references}  %

\appendix

\section{Appendix}

\subsection{Integrated Autocorrelation Time}

The number of samples between independent samples is determined by the integrated autocorrelation time (IACT)
\eq{IACT}{
    \widehat{\tau} = 1 + 2 \sum_{\tau=1}^M \widehat{\rho}_Q(\tau)
}
where the normalized autocorrelation function is estimated as
\eq{NormAutoCorr}{
    \widehat{\rho}_Q(\tau) = \dfrac{1}{N-\tau} \sum_{i=1}^M \dfrac{\rbrac{Q^{(i)} - \widehat{\mu}_Q} \rbrac{Q^{(i+\tau)} - \widehat{\mu}_Q}}{\widehat{\sigma}^2_Q}
}
with mean $\widehat{\mu}_Q$ and variance $\widehat{\sigma}_Q$ of $\cbrac{Q^{(i)}}_{i=1}^N$ and $M \ll N$.

\subsection{Number of Effective Samples} \label{Appen:NEffective}

Consider the effective cost of generating an independent sample on level $\ell$
\eq{EffecCost}{
    \overline{C_\ell} := \ceil*{\tau_\ell} \rbrac{C_\ell + \sum_{i=1}^{\ell-1} T_i C_i}
}
where $T_i = \prod_{j=0}^{i} \ceil*{\tau_j}$ is the effective subsampling rate of the coarser levels and the following bound on error of the multilevel estimator
\expression{
    \sum_{\ell=0}^L \dfrac{\sigma^2_\ell}{N_\ell} \leq \dfrac{\epsilon^2}{2}
}
where $\sigma^2_\ell = \Var[\pi_\ell,\pi_{\ell-1}]{Y_\ell}$ is the sample variance with respect to the joint distributions $\pi_\ell$ and $\pi_{\ell-1}$. The optimization problem for determining the optimal number of independent samples reads
\aligneq{Neff:OptimProbl}{
    \min &~\sum_{\ell=0}^L \overline{C_\ell} N_\ell \\
    \mathrm{s.t}& ~\sum_{\ell=0}^L \dfrac{\sigma^2_\ell}{N_\ell} \leq \dfrac{\epsilon^2}{2}
}
with the Lagrangian
\expression{
    \mathbb{L} = \sum_{\ell=0}^L \overline{C_\ell} N_\ell + \lambda \rbrac{\sum_{\ell=0}^L \dfrac{\sigma^2_\ell}{N_\ell} - \dfrac{\epsilon^2}{2}}
}
and the stationarity condition 
\eq{Neff:Stationarity}{
    \dfdx{\mathbb{L}}{N_\ell} = \overline{C_\ell} - \lambda \dfrac{\sigma^2_\ell}{N^2_\ell} = 0 \rightarrow \sqrt{\dfrac{\lambda \sigma^2_\ell}{\overline{C_\ell}}} = N_\ell
}
Subsituting \eqref{Neff:Stationarity} into the primal feasibility condition, where the inequality constraint is treated as a binding, equality constraint
\eq{Neff:PrimalFeasibility}{
    \sum_{\ell=0}^L \dfrac{\sigma^2_\ell}{N_\ell} - \dfrac{\epsilon^2}{2} = 0
}
yields
\expression{
    \sum_{\ell=0}^L \sqrt{\dfrac{\sigma^2_\ell \overline{C_\ell}}{\lambda}} - \dfrac{\epsilon^2}{2} = 0 \rightarrow \sqrt{\lambda} = \dfrac{2}{\epsilon^2} \sum_{\ell=0}^L \sqrt{\sigma^2_\ell \overline{C_\ell}}
}
Subsituting back into \eqref{Neff:Stationarity} yields the desired expression for the number of effective samples as provided in \eqref{EffectiveSampleSize}.

\end{document}